\documentclass[preprint,12pt]{elsarticle}
\usepackage{amssymb}
\usepackage{amsmath}
\usepackage{times}
\usepackage{indentfirst}
\usepackage{graphicx}
\newtheorem{thm}{Theorem}
\newtheorem{theorem}[thm]{Theorem}
\newtheorem{definition}[thm]{Definition}
\newtheorem{proposition}[thm]{Proposition}

\newtheorem{lemma}[thm]{Lemma}
\newtheorem{example}[thm]{Example}
\newdefinition{remark}{Remark}
\newproof{proof}{Proof}

\begin{document}

\begin{frontmatter}
\title{Condition for neighborhoods in covering based rough sets to form a partition}         
\author{Hua Yao}
\author{William Zhu\corref{cor1}}
\ead{williamfengzhu@gmail.com}
\cortext[cor1]{Corresponding author}

\address{Lab of Granular Computing,\\
Zhangzhou Normal University, Zhangzhou, China 363000}

\begin{abstract}
Neighborhood is an important concept in covering based rough sets. That under what condition neighborhoods form a partition is a meaningful issue induced by this concept. Many scholars have paid attention to this issue and presented some necessary and sufficient conditions. However, there exists one common trait among these conditions, that is they are established on the basis of all neighborhoods have been obtained. In this paper, we provide a necessary and sufficient condition directly based on the covering itself. First, we investigate the influence of that there are reducible elements in the covering on neighborhoods. Second, we propose the definition of uniform block and obtain a sufficient condition from it. Third, we propose the definitions of repeat degree and excluded number. By means of the two concepts, we obtain a necessary and sufficient condition for neighborhoods to form a partition. In a word, we have gained a deeper and more direct understanding of the essence over that neighborhoods form a partition.
\end{abstract}

\begin{keyword}
Neighborhood; Reducible element; Membership repeat degree; Excluded number.
\end{keyword}

\end{frontmatter}
\section{Introduction}
Rough set theory, proposed by Pawlak~\cite{Pawlak82Rough,Pawlak91Rough}, is an extension of set theory for the
study of intelligent systems characterized by insufficient and incomplete information. In theory, rough sets have been connected with matroids~\cite{TangSheZhu12matroidal,WangZhuZhuMin12matroidalstructure}, lattices~\cite{Dai05Logic,EstajiHooshmandaslDavvaz12Roughappliedtolattice,Liu08Generalized,WangZhu11Quantitative}, hyperstructure theory~\cite{YamakKazanciDavvaz11Softhyperstructure},
topology~\cite{Kondo05OnTheStructure,LashinKozaeKhadraMedhat05Rough,Zhu07Topological}, fuzzy sets~\cite{KazanciYamakDavvaz08TheLower,WuLeungMi05OnCharacterizations}, and so on. Rough set theory is built on an equivalence relation, or to say, on a partition. But equivalence relation or partition is still restrictive for many applications. To address
this issue, several meaningful extensions to equivalence relation
have been proposed. Among them, Zakowski has used coverings
of a universe for establishing the covering based rough set theory~\cite{Zakowski83Approximations}. Many scholars have done deep researches on this theory~\cite{BonikowskiBryniarskiWybraniecSkardowska98Extensions,Bryniarski89ACalculus,ZhuWang03Reduction}, and some basic results have been presented.

Neighborhood is an important concept in covering based rough set theory. Many scholars have studied it from different perspectives. Lin augmented the relational database with neighborhood~\cite{Lin88Neighborhoodsystems}. Yao presented a framework for the formulation, interpretation, and comparison of neighborhood systems and rough set approximations~\cite{Yao98Relational}. By means of consistent function based on the concept of neighborhood, Wang et al.~\cite{WangChenSunHu12Communication} dealt with information systems through covering based rough sets. Furthermore, the concept of neighborhood itself has produced lots of meaningful issues as well, and under what condition neighborhoods form a partition is one of them. Many scholars have focused on this issue and conducted some researches on it~\cite{FanHuXiaoZhang12Study,QinGaoPei07OnCovering,YunGeBai11Axiomatization}. Different scholars provided different sufficient and necessary conditions respectively. However, there is a common trait among these necessary and sufficient conditions, that is the neighborhoods had been calculated out before the necessary and sufficient condition was presented. For example, Yun et al.~\cite{YunGeBai11Axiomatization} studied the conditions for neighborhoods to form a partition from the viewpoint of operators, while the operators were defined by all neighborhoods. If all the neighborhoods have been calculated out, then whether or not the neighborhoods form a partition is already clear. So it is necessary to seek condition for neighborhoods to form a partition directly based on the covering itself.

In this paper, we provide a necessary and sufficient condition directly based on the covering itself. First, we investigate the influence of that there are reducible elements in the covering on neighborhoods. We prove that the reducible elements in the covering have no influence on the neighborhoods induced by the covering. Second, we propose the definition of uniform block and obtain a sufficient condition from it. We also give a counter-example to prove the condition is not necessary. Third, we propose the definitions of repeat degree and excluded number, and obtain some properties of them. By means of the two concepts and their properties, we obtain a necessary and sufficient condition for neighborhoods to form a partition. This necessary and sufficient condition for neighborhoods to form a partition does not involve in any lower or upper approximations, but the covering itself.

The remainder of this paper is organized as follows. In Section~\ref{S:Basic definitions}, we review the relevant concepts and introduce some existing results. In Section~\ref{S:Two sufficient conditions}, we give two sufficient conditions for neighborhoods to form a partition. In Section~\ref{S:The main theorem in this paper}, we present a sufficient and necessary condition. Section~\ref{S:Conclusions} concludes this paper and points out further works.

\section{Preliminaries}
\label{S:Basic definitions}
We introduce the definitions of covering and partition at first.

\begin{definition}(Covering)
Let $U$ be a universe of discourse and $\mathbf{C}$ a family of subsets of $U$. If $\emptyset\notin\mathbf{C}$, and $\cup \mathbf{C}=U$, then $\mathbf{C}$ is called a covering of $U$. Every element of $\mathbf{C}$ is called a covering block.
\end{definition}

In the following discussion, unless stated to the contrary, the universe of discourse $U$ is considered to be
finite and nonempty.

\begin{definition}(Partition)
Let $U$ be a universe and $\mathbf{P}$ a family of subsets of $U$. If $\emptyset\notin\mathbf{P}$, and $\cup \mathbf{P}=U$, and for any $K,L\in\mathbf{P}$, $K\cap L=\emptyset$, then $\mathbf{P}$ is called a partition of $U$. Every element of $\mathbf{P}$ is called a partition block.
\end{definition}

It is clear that a partition of $U$ is certainly a covering of $U$, so the concept of covering is an extension of the concept of partition.

In the following, we introduce the definitions of neighborhood and neighborhoods, the two main concepts which will be discussed in this paper.

\begin{definition}(Neighborhood~\cite{Lin88Neighborhoodsystems})
Let $\mathbf{C}$ be a covering of $U$. For any $x\in U$, $N(x)=\cap\{K\in\mathbf{C}|x\in K\}$ is called the neighborhood of $x$.
\end{definition}

In the following proposition, we introduce relationships between the neighborhoods of any two elements of a universe.

\begin{proposition}~\cite{WangChenSunHu12Communication}
Let $\mathbf{C}$ be a covering of $U$. For any $x,y\in U$, if $y\in N(x)$, then $N(y)\subseteq N(x)$. So if $y\in N(x)$ and $x\in N(y)$, then $N(x)=N(y)$.
\end{proposition}

\begin{definition}~\cite{WangChenSunHu12Communication}
Let $\mathbf{C}$ be a covering of $U$. $Cov(\mathbf{C})=\{N(x)|x\in U\}$ is called the neighborhoods induced by $\mathbf{C}$.
\end{definition}

By the definition of $Cov(\mathbf{C})$, we see that $Cov(\mathbf{C})$ is still a covering of universe $U$. Papers~\cite{FanHuXiaoZhang12Study,QinGaoPei07OnCovering,YunGeBai11Axiomatization} provided some necessary and sufficient conditions for $Cov(\mathbf{C})$ to form a partition. In the following, we introduce the definition of covering approximation space and three conditions for $Cov(\mathbf{C})$ to form a partition.

\begin{definition}(Covering approximation space~\cite{ZhuWang03Reduction})
Let $U$ be a universe and $\mathbf{C}$ a covering of $U$. The ordered
pair $(U,\mathbf{C})$ is called a covering approximation space.
\end{definition}

\begin{proposition}~\cite{QinGaoPei07OnCovering}
Let $(U,\mathbf{C})$ be a covering approximation space. Then $Cov(\mathbf{C})$ forms a partition of $U$ if and only if for any $X\subseteq U$, $\underline{\mathbf{C}_{4}}(X)=\underline{\mathbf{C}_{2}}(X)$, where $\underline{\mathbf{C}_{4}}(X)=\{x\in U|\forall u(x\in N(u)\rightarrow N(u)\subseteq X)\}$, $\underline{\mathbf{C}_{2}}(X)=\{x\in U|N(x)\subseteq X\}$.

\end{proposition}

\begin{proposition}~\cite{YunGeBai11Axiomatization}
Let $(U,\mathbf{C})$ be a covering approximation space. Then $Cov(\mathbf{C})$ forms a partition of $U$ if and only if for any $X\subseteq U$, $\overline{\mathbf{C}_{3}}(\underline{\mathbf{C}_{3}}(X))=\underline{\mathbf{C}_{3}}(X)$, where $\underline{\mathbf{C}_{3}}(X)=\{x\in U|N(x)\subseteq X\}$, $\overline{\mathbf{C}_{3}}(X)=\{x\in U|N(x)\cap X\neq\emptyset\}$.
\end{proposition}

\begin{proposition}~\cite{FanHuXiaoZhang12Study}
Let $(U,\mathbf{C})$ be a covering approximation space.
Then $Cov(\mathbf{C})$ forms a partition of $U$ if and only if for any x,
$\overline{\mathbf{C}}(\{x\})=N(x)$, where $\overline{\mathbf{C}}(X)=\{x\in U|\forall K\in\mathbf{C}(x\in K\rightarrow K\cap X\neq\emptyset)\}$.
\end{proposition}

From the above three propositions, we can see that there are some special properties on covering approximation operators when $Cov(\mathbf{C})$ forms a partition. There are some more in-depth discussions in Paper~\cite{FanHuXiaoZhang12Study,QinGaoPei07OnCovering,YunGeBai11Axiomatization} regarding this issue. However, we can see that every $N(x)$ was used directly or indirectly in the description of the necessary and sufficient conditions. In fact, if all the $N(x)$ have been calculated out, then whether or not the neighborhoods form a partition is already clear. In the remainder of this paper, we will present a necessary and sufficient condition directly based on the covering itself.

\section{Two sufficient conditions}
\label{S:Two sufficient conditions}
In this section, we present two sufficient conditions for neighborhoods to form a partition. The concept of reducible element is needed for the description of one sufficient condition.

\begin{definition}(Reducible element~\cite{ZhuWang03Reduction})
Let $\mathbf{C}$ be a covering of a universe $U$ and $K\in \mathbf{C}$. If $K$ is a union of some blocks in $\mathbf{C}-\{K\}$, we say $K$ is a reducible element of $\mathbf{C}$, otherwise $K$ is an irreducible element of $\mathbf{C}$.
\end{definition}

\begin{proposition}~\cite{ZhuWang03Reduction}
\label{P:C-Kisstillacovering}
Let $\mathbf{C}$ be a covering of a universe $U$. If $K$ is a reducible element of $\mathbf{C}$, $\mathbf{C}-\{K\}$ is still a covering of $U$.
\end{proposition}

\begin{proposition}~\cite{ZhuWang03Reduction}
\label{P:it is a reducible element of}
Let $\mathbf{C}$ be a covering of a universe $U$, $K\in \mathbf{C}$, $K$ is a reducible element of $\mathbf{C}$, and $K_{1}\in \mathbf{C}-\{K\}$, then $K_{1}$ is a reducible element of $\mathbf{C}$ if and only if it is a reducible element of $\mathbf{C}-\{K\}$.
\end{proposition}

Proposition~\ref{P:C-Kisstillacovering} guarantees that after deleting a reducible element in a covering,
it is still a covering, whereas Proposition~\ref{P:it is a reducible element of} shows that deleting a reducible element
in a covering will not generate any new reducible elements or make other
originally reducible elements become irreducible elements of the new covering. So, we can compute the reduct of a covering of a universe $U$ by deleting all
reducible elements in the same time, or by deleting one reducible element in a
step.
\begin{definition}(Reduct~\cite{ZhuWang03Reduction})
Let $\mathbf{C}$ be a covering of a universe $U$ and $D$ a subset of $\mathbf{C}$. If $\mathbf{C}-D$ is the set of all reducible elements of $\mathbf{C}$, then $D$ is called the reduct of $\mathbf{C}$, and is denoted  as $reduct(\mathbf{C})$.
\end{definition}

The following proposition indicates that deleting the reducible elements from the covering has no influence on the neighborhoods.

\begin{proposition}
Let $\mathbf{C}$ be a covering of a universe $U$, then
\begin{center}
$Cov(\mathbf{C})=Cov(reduct(\mathbf{C}))$.
\end{center}
\end{proposition}

\begin{proof}

We prove this proposition using induction on $m(m\geq1)$, the amount of reducible elements.

Assume that the proposition is true for that the amount of reducible elements is less than $m$.

Assume that the amount of reducible elements is equal to $m$ and $K$ is a reducible element of $\mathbf{C}$. By Proposition~\ref{P:C-Kisstillacovering}, we have that $\mathbf{C}-\{K\}$ is still a covering of $U$, and there exists a set $L\subseteq\mathbf{C}-\{K\}$, such that $K=\cup L$. For any $x\in U$, we denote the neighborhood of $x$ induced by covering $\mathbf{C}$ as $N_{c}(x)$, denote the neighborhood of $x$ induced by covering $\mathbf{C}-\{K\}$ as $N_{c-\{K\}}(x)$.

For any $x\in U$, it follows that $x\notin K$ or $x\in K$. If $x\notin K$, $N_{c-\{K\}}(x)=N_{c}(x)$ holds obviously. If $x\in K$, by $K=\cup L$, we have that there exists $P\in L$, i.e. $P\subset K$ such that $x\in P$.

Let $\{A|A\in \mathbf{C}-\{K\}\wedge x\in A\}=W$. It is clear that $P\in W$. Therefore $N_{c-\{K\}}(x)=\cap W$, $N_{c}(x)=(\cap W)\cap K$. By $P\in W$, we have $\cap W\subseteq P\subset K$, thus $(\cap W)\cap K=\cap W$, then $N_{c-\{K\}}(x)=N_{c}(x)$. Taking into account the arbitrariness of $x$, we have that $Cov(\mathbf{C})=Cov(\mathbf{C}-\{K\})$.

By Proposition~\ref{P:it is a reducible element of}, we see that there are $m-1$ reducible elements in set $\mathbf{C}-\{K\}$. By the induction hypothesis, we have that $Cov(\mathbf{C}-\{K\})=Cov(reduct(\mathbf{C}-\{K\}))$. Again, by Proposition~\ref{P:it is a reducible element of}, we have that $reduct(\mathbf{C}-\{K\})=reduct(\mathbf{C})$. Integrating the results as above, we have $Cov(\mathbf{C})=Cov(reduct(\mathbf{C}))$.

This completes the proof.
\end{proof}

If the covering $\mathbf{C}$ of a universe $U$ is a partition, it is clear that $Cov(\mathbf{C})=\mathbf{C}$ is a partition. Therefore, we obtain the following theorem.

\begin{theorem}
\label{T:sufficient condition 1}
Let $\mathbf{C}$ be a covering of a universe $U$. If $reduct(\mathbf{C})$ is a partition, then $Cov(\mathbf{C})$ forms a partition.
\end{theorem}

The following counter-example indicates that the condition is not necessary.

\begin{example}
Let $U=\{1,2,3,4\}$, $\mathbf{C}=\{K_{1},K_{2},K_{3},K_{4}\}$, where $K_{1}=\{1,2,3\}$, $K_{2}=\{1,2\}$, $K_{3}=\{3,4\}$, $K_{4}=\{4\}$, then $reduct(\mathbf{C})=\mathbf{C}$ is not a partition. But by $N(1)=N(2)=\{1,2\}$, $N(3)=\{3\}$, $N(4)=\{4\}$, we have that $Cov(\mathbf{C})=\{\{1,2\},\{3\},\{4\}\}$ is a partition.
\end{example}

Now, we give some new definitions and then give the other sufficient condition for neighborhoods to form a partition.

\begin{definition}(Membership repeat degree)
\label{D:Membership repeat degree}
Let $\mathbf{C}$ be a covering of a universe $U$. We define a function $\partial_{\mathbf{C}}:U\rightarrow N^{+}$, $\partial_{\mathbf{C}}(x)=|\{K\in\mathbf{C}|x\in K\}|$, and call $\partial_{\mathbf{C}}(x)$ the membership repeat degree of $x$ with respect to covering $\mathbf{C}$. When the covering is clear, we omit the lowercase $\mathbf{C}$ for the function.
\end{definition}

That an element $x$ of $U$ has the membership repeat degree of $\partial(x)$ means there are $\partial(x)$  blocks in covering  $\mathbf{C}$ that contain element $x$.

To illustrate the above definition, let us see an example.

\begin{example}
Let $U=\{1,2,3\}$, $\mathbf{C}=\{K_{1},K_{2}\}$, where $K_{1}=\{1,2\}$, $K_{2}=\{2,3\}$. Then $\{K\in\mathbf{C}|1\in K\}=\{K_{1}\}$, $\{K\in\mathbf{C}|2\in K\}=\{K_{1},K_{2}\}$, $\{K\in\mathbf{C}|3\in K\}=\{K_{2}\}$, thus $\partial(1)=|\{K_{1}\}|=1$, $\partial(2)=|\{K_{1},K_{2}\}|=2$, $\partial(3)=|\{K_{2}\}|=1$.
\end{example}

\begin{definition}(Uniform block)
Let $\mathbf{C}$ be a covering of a universe $U$. For any  $K\in\mathbf{C}$, $K$ is called a uniform block with respect to covering $\mathbf{C}$ if and only if all the elements belonging to $K$ have the same membership repeat degree.
\end{definition}

To illustrate the above definition, let us see an example.

\begin{example}
Let $U=\{1,2,3,4\}$, $\mathbf{C}=\{K_{1},K_{2},K_{3}\}$, where $K_{1}=\{1,2\}$, $K_{2}=\{2,3,4\}$, $K_{3}=\{3,4\}$. We have $\partial(1)=1$, $\partial(2)=\partial(3)=\partial(4)=2$, thus $K_{2}$ and $K_{3}$ are uniform blocks, but $K_{1}$ is not a uniform block.
\end{example}

By the definition of uniform block, we obtain the following theorem.

\begin{theorem}
\label{T:sufficient condition 2}
If all the blocks of covering $\mathbf{C}$ are uniform blocks, then $Cov(\mathbf{C})$ forms a partition.
\end{theorem}

\begin{proof}

We use an indirect proof. Suppose $Cov(\mathbf{C})$ is not a partition, then there exists at least one $x\in U$, such that $|\{K\in Cov(\mathbf{C})|x\in K\}|>1$. Since it is clear that $x\in N(x)$, so we suppose there is another block $N(y)\in Cov(\mathbf{C})$, such that $x\in N(y)$, where $y\neq x$, and $y\notin N(x)$, for if it is not so, we will obtain $N(x)=N(y)$. By $x\in N(y)$, we have $\forall L((L\in\mathbf{C}\wedge y\in L)\rightarrow x\in L)$. By $y\notin N(x)$, we have $\exists K(K\in\mathbf{C}\wedge x\in K\wedge y\notin K)$.
Integrating the two results as above, we have $\partial(x)>\partial(y)$.
By $x\in N(y)$, we have $\exists M(M\in\mathbf{C}\wedge y\in M\wedge x\in M)$, thus we see $M$ is not a uniform block. This is a contradiction to the hypothesis.

This completes the proof.
\end{proof}

The following counter-example indicates that the condition is not necessary.

\begin{example}
Let $U=\{1,2,3,4\}$, $\mathbf{C}=\{K_{1},K_{2},K_{3},K_{4},K_{5}\}$, where $K_{1}=\{1,2,3\}$, $K_{2}=\{1,2\}$, $K_{3}=\{3,4\}$, $K_{4}=\{3\}$, $K_{5}=\{4\}$, then $\partial(3)=3$, $\partial(4)=2$, so $K_{3}$  is not a uniform block. But $N(1)=N(2)=\{1,2\}$, $N(3)=\{3\}$, $N(4)=\{4\}$, thus $Cov(\mathbf{C})=\{\{1,2\},\{3\},\{4\}\}$ is a partition.
\end{example}

The sufficient conditions in Theorem~\ref{T:sufficient condition 1} and in Theorem~\ref{T:sufficient condition 2} are independent from each other. To illustrate it, let us see the following two examples.

\begin{example}
Let $U=\{1,2,3\}$, $\mathbf{C}=\{\{1\},\{2\},\{3\},\{1,2\},\{1,3\}\}$, then $reduct(\mathbf{C})\\
=\{\{1\},\{2\},\{3\}\}$  is a partition. But $\partial(1)=3$, $\partial(2)=\partial(3)=2$, so both $\{1,2\}$ and $\{1,3\}$ are not uniform blocks.
\end{example}

\begin{example}
Let $U=\{1,2,3\}$, $\mathbf{C}=\{\{1,2\},\{1,3\},\{2,3\}\}$, then all of $\{1,2\}$, $\{1,3\}$, $\{2,3\}$ are uniform blocks. But $reduct(\mathbf{C})=\{\{1,2\},\{1,3\},\{2,3\}\}$ is not a partition.
\end{example}

\section{A sufficient and necessary condition }
\label{S:The main theorem in this paper}
In this section, we propose some new concepts. By means of them, we obtain a necessary and sufficient condition for neighborhoods to form a partition.

\begin{definition}(Common block repeat degree)
\label{D:Common block repeat degree}
Let $\mathbf{C}$ be a covering of a universe $U$. We define a function $\lambda_{\mathbf{C}}: U\times U\rightarrow N, \lambda_{\mathbf{C}}((x,y))=|\{K\in\mathbf{C}|\{x,y\}\subseteq K\}|$. We write $\lambda_{\mathbf{C}}((x,y))$ as  $\lambda_{\mathbf{C}}(x,y)$ for short, and for any $x,y\in U$, we call $\lambda_{\mathbf{C}}(x,y)$ the common block repeat degree of binary group $(x,y)$ with respect to covering $\mathbf{C}$. When the covering is clear, we omit the lowercase $\mathbf{C}$ for the function.
\end{definition}

That a binary group $(x,y)$ of universe $U$ has the common block repeat degree of $\lambda(x,y)$ with respect to covering $\mathbf{C}$ means there are $\lambda(x,y)$ blocks in covering $\mathbf{C}$ that contain element $x$ and $y$  simultaneously.

To illustrate the above definition, let us see an example.

\begin{example}
Let $U=\{1,2,3,4\}$, $\mathbf{C}=\{K_{1},K_{2},K_{3}\}$, where $K_{1}=\{1,2\}$, $K_{2}=\{2,3,4\}$, $K_{3}=\{3,4\}$. Then $\lambda(1,2)=\lambda(2,3)=\lambda(2,4)=1$, $\lambda(1,3)=\lambda(1,4)=0$, $\lambda(3,4)=2$.
\end{example}

The common block repeat degree $\lambda(x,y)$ has some properties as follows.

\begin{proposition}
(1) $\lambda(x,y)=\lambda(y,x)$; (2) $\lambda(x,y)\leq min(\partial(x),\partial(y))$.
\end{proposition}

\begin{proof}
It follows easily from Definition~\ref{D:Membership repeat degree} and Definition~\ref{D:Common block repeat degree}.
\end{proof}

\begin{proposition}
\label{P:relationship between two types of repeat degree}
Let $\mathbf{C}$ be a covering of a universe $U$. For any $x,y\in U$,
$\{K\in\mathbf{C}|x\in K\}=\{K\in\mathbf{C}|\{x,y\}\subseteq K\}\Leftrightarrow\partial(x)=\lambda(x,y)$.
\end{proposition}

\begin{proof}

$(\Rightarrow)$: It is straightforward.\\
$(\Leftarrow)$: It is clear that $\{K\in\mathbf{C}|\{x,y\}\subseteq K\}\subseteq\{K\in\mathbf{C}|x\in K\}$. If $\{K\in\mathbf{C}|\{x,y\}\subseteq K\}\neq\{K\in\mathbf{C}|x\in K\}$, therefore  $\{K\in\mathbf{C}|\{x,y\}\subseteq K\}$ is the proper subset of $\{K\in\mathbf{C}|x\in K\}$. Taking into account the finiteness of set $\{K\in\mathbf{C}|x\in K\}$, we have $|\{K\in\mathbf{C}|\{x,y\}\subseteq K\}|<|\{K\in\mathbf{C}|x\in K\}|$, thus $\lambda(x,y)<\partial(x)$. This is a contradiction to that $\partial(x)=\lambda(x,y)$.

This completes the proof.
\end{proof}

\begin{definition}(Excluded number)
Let $\mathbf{C}$ be a covering of a universe $U$. For any $x,y\in U$, we call $f_{y}(x)=\partial(x)-\lambda(x,y)$ the $y$ excluded number of $x$.
\end{definition}

\begin{lemma}
\label{lemma}
Let $\mathbf{C}$ be a covering of a universe $U$. For any $x,y\in U$, $y\in N(x)$ if and only if $f_{y}(x)=0$.
\end{lemma}

\begin{proof}
According to Proposition~\ref{P:relationship between two types of repeat degree}, we have\\
$y\in N(x)\Leftrightarrow\forall K((K\in\mathbf{C}\wedge x\in K)\rightarrow(y\in K))\Leftrightarrow\forall K((K\in\mathbf{C}\wedge x\in K)\rightarrow(K\in\mathbf{C}\wedge \{x,y\}\subseteq K))\Leftrightarrow\forall K((K\in\mathbf{C}\wedge x\in K)\leftrightarrow(K\in\mathbf{C}\wedge \{x,y\}\subseteq K))\Leftrightarrow\forall K(K\in\mathbf{C}\wedge x\in K)\leftrightarrow\forall K(K\in\mathbf{C}\wedge \{x,y\}\subseteq K)\Leftrightarrow\{K\in\mathbf{C}|x\in K\}=\{K\in\mathbf{C}|\{x,y\}\subseteq K\}\Leftrightarrow\partial(x)=\lambda(x,y)\Leftrightarrow f_{y}(x)=0$.

This completes the proof.
\end{proof}

Now, we present a necessary and sufficient condition for neighborhoods to form a partition, the main theorem in this paper.

\begin{theorem}
Let $\mathbf{C}$ be a covering of a universe $U$, $Cov(\mathbf{C})$ forms a partition if and only if for any $x,y\in U$, $f_{y}(x)=f_{x}(y)=0$, or $f_{y}(x)\neq0$ and $f_{x}(y)\neq0$.
\end{theorem}

\begin{proof}
$(\Leftarrow)$: We use an indirect proof. Suppose $Cov(\mathbf{C})$ is not a partition, then there exists at least one $x\in U$, such that $|\{K\in Cov(\mathbf{C})|x\in K\}|>1$. For it is clear that $x\in N(x)$, so we suppose there is another $N(y)\in Cov(\mathbf{C})$, such that $x\in N(y)$, where $y\neq x$, and $y\notin N(x)$, for if it is not so, we will obtain $N(x)=N(y)$. By $x\in N(y)$ and Lemma~\ref{lemma}, we have $f_{x}(y)=0$. By $y\notin N(x)$ and Lemma~\ref{lemma}, we have $f_{y}(x)\neq0$. This is a contradiction to the hypothesis.

$(\Rightarrow)$: We use an indirect proof. Suppose there are $x,y\in U$, such that $f_{y}(x)=0$, $f_{x}(y)\neq0$. By Lemma~\ref{lemma}, we have $y\in N(x)$, $x\notin N(y)$. Thus $N(x)\neq N(y)$, so there are two blocks $N(x)$ and $N(y)$ in  $Cov(\mathbf{C})$ that contain the element $y$, so $Cov(\mathbf{C})$ is not a partition. This is a contradiction to the hypothesis. Similarly, we obtain a contradiction to the hypothesis when $f_{y}(x)\neq0$ and $f_{x}(y)=0$.

This completes the whole proof.
\end{proof}

\section{Conclusions }
\label{S:Conclusions}
Neighborhood is an important concept in covering based rough sets, and through some concepts based on neighborhood and neighborhoods such as consistent function, we may find new connections between covering based rough sets and information systems, so it is necessary to study the properties of neighborhood and neighborhoods themselves. That under what condition neighborhoods form a partition is one of the fundamental issues induced by the two concepts. There are still many issues induced by neighborhood and neighborhoods to solve. We will continually focus on these issues in our following research.

\section*{Acknowledgments}
This work is supported in part by the National Natural Science Foundation of China under Grant No. 61170128, the Natural Science Foundation of Fujian Province, China, under Grant Nos. 2011J01374 and 2012J01294, and the Science and Technology Key Project of Fujian Province, China, under Grant No. 2012H0043.


\end{document}